\documentclass{SCAEOL}
\numberwithin{equation}{section}
\usepackage{amsfonts,amssymb}
\usepackage{multirow}
\usepackage{threeparttable}
\begin{document}

\Year{2019} %
\Month{September}
\Vol{60} %
\No{1} %
\BeginPage{1} %
\EndPage{XX} %
\AuthorMark{Guo L {\it et al.}}
\ReceivedDay{November, 2018}
\AcceptedDay{July, 2019}

\title{Subsampling Bias and The Best-Discrepancy Systematic Cross Validation}{This paper has been accepted for publication by SCIENCE CHINA Mathematics}

\author[1]{Liang Guo}{}
\author[2]{Jianya Liu}{Corresponding author}
\author[3,4]{Ruodan Lu}{}

\address[{\rm1}]{Data Science Institute, Shandong University, Weihai {\rm 264209}, China;}
\address[{\rm2}]{Data Science Institute,  Shandong University, Jinan {\rm 250100}, China;}
\address[{\rm3}]{School of Architecture, Building and Civil Engineering, Loughborough University, Loughborough {\rm LE11 3TU}, UK;}
\address[{\rm4}]{Darwin College, University of Cambridge, Cambridge {\rm CB3 9EU}, UK}
\Emails{liangguo@sdu.edu.cn,jyliu@sdu.edu.cn, rl508@cam.ac.uk}\maketitle


 {\begin{center}
\parbox{14.5cm}{\begin{abstract}
Statistical machine learning models should be evaluated and validated before putting to work. Conventional $k$-fold Monte Carlo Cross-Validation (MCCV) procedure uses a pseudo-random sequence to partition instances into $k$ subsets, which usually causes subsampling bias, inflates generalization errors and jeopardizes the reliability and effectiveness of cross-validation. Based on ordered systematic sampling theory in statistics and low-discrepancy sequence theory in number theory, we propose a new $k$-fold cross-validation procedure by replacing a pseudo-random sequence with a best-discrepancy sequence, which ensures low subsampling bias and leads to more precise \textit{Expected-Prediction-Error} (\textit{EPE}) estimates. Experiments with 156 benchmark datasets and three classifiers (logistic regression, decision tree and na{\"{\i}}ve bayes) show that in general, our cross-validation procedure can extrude subsampling bias in the MCCV by lowering the \textit{EPE} around 7.18\% and the variances around 26.73\%. In comparison, the stratified MCCV can reduce the \textit{EPE} and variances of the MCCV around 1.58\% and 11.85\% respectively. The Leave-One-Out (LOO) can lower the \textit{EPE} around 2.50\% but its variances are much higher than the any other CV procedure. The computational time of our cross-validation procedure is just 8.64\% of the MCCV, 8.67\% of the stratified MCCV and 16.72\% of the LOO. Experiments also show that our approach is more beneficial for datasets characterized by relatively small size and large aspect ratio. This makes our approach particularly pertinent when solving bioscience classification problems. Our proposed systematic subsampling technique could be generalized to other machine learning algorithms that involve random subsampling mechanism.\vspace{-3mm}
\end{abstract}}\end{center}}

 \keywords{Subsampling Bias, Cross Validation, Systematic Sampling, Low-Discrepancy Sequence, Best-Discrepancy Sequence}

 \MSC{62-07; 11J71; 62G09; 68T05}

\renewcommand{\baselinestretch}{1.2}
\begin{center} \renewcommand{\arraystretch}{1.5}
{\begin{tabular}{lp{0.8\textwidth}} \hline \scriptsize
{\bf Citation:}\!\!\!\!&\scriptsize Guo L, Liu J, Lu R.  Sci China Math, 2019. doi: xxxxxxxxx\vspace{1mm}
\\
\hline
\end{tabular}}\end{center}

\baselineskip 11pt\parindent=10.8pt  \wuhao

\section{Introduction}
All statistical machine learning processes start with a basic truism: garbage in, garbage out. This means that nonsense input data usually produce nonsense conclusions. However, ``garbage in, garbage out" is not only about data quality, but also about the way in which data scientists evaluate and choose statistical machine learning models. Such evaluation often relies on a Cross-Validation (CV) to estimate model generalization performance metrics. The holdout procedure is the simplest kind of CV, in which a data set is split into a training set and a test set. However, the holdout is not reliable as the evaluation depends heavily on which data points end up in the training set and which in the test set \cite{Kohavi1995}.  Thus, the evaluation may be significantly different depending on how the subsampling is made \cite{Kohavi1995}. The $k$-fold Monte-Carlo CV (MCCV), the most popular CV procedure, is an effective way to reduce the subsampling bias in the holdout procedure. Given a dataset of $n$ instances, an MCCV procedure is to randomly generate $k$ subsets, each of which contains $m$ instances, where $n=km$. That is to say, each instance is first assigned an index between 1 and $n$ to construct a subsampling frame. Using the Simple Random Subsampling (SRS) method, the MCCV procedure then uses a Pseudo-Random Number Generator (PRNG) to generate a random integer sequence. This sequence consists of $n$ distinct, randomly ordered integers between 1 and $n$, each of which refers to an element of the subsampling frame. Next, the random integer sequence is split into $k$ equal-sized parts and the corresponding instances are partitioned into $k$ subsets. Subsequently, $k$ iterations of training and testing are performed such that in each iteration, a different subset is held out to estimate the \textit{EPE} while the remaining $k-1$ subsets are used for training \cite{Stone1974}.

However, relying on the SRS technique to subsample instances to training and test sets not only generates high computational costs, but also poses risks \cite{Braga-Neto2004,Molinaro2005}. The inductive bias of a model and the subsampling bias in the SRS may be confounded and lead to inflated \textit{EPE}s \cite{Molinaro2005}. This is especially true for small, imbalanced, and high dimensional datasets \cite{Fu2005}. The random integer sequence used in the SRS poses a threat to the MCCV's validity and reliability, as the sequence may seem to be random, but not truly random and not truly uniformly distributed. This is because the elements in such sequence often clump at certain arbitrary intervals \cite{Gentle2003}. When performing SRS, similar instances could be included in one subset, meaning that information may be over- (clusters) or under-represented (voids) in that subset. This is just like the random shuffle function of \textit{Spotify} (a music player App) cannot ensure songs from a given artist are spread throughout the playlist. To avoid the risk of obtaining dissimilar subsets, the MCCV procedure should be repeated $p$ times (usually $p=50$) with different random seeds to circumvent the risk of obtaining greatly dissimilar subsets \cite{Molinaro2005}, as there are trillions of ways to partition $n$ instances into $k$ subsets. 

The reliability of the MCCV increases as the $k$ and $p$ increase. When $k=n$, the MCCV becomes the Leave-One-Out CV (LOO). In practice, the LOO is not widely used because it is more computationally demanding and has a larger variance than the MCCV \cite{Kohavi1995}. The stratified MCCV procedure is another mean to alleviate the subsampling bias issue \cite{Pedregosa2011}. The stratified MCCV procedure starts by dividing a dataset into non-overlapping strata based on its class label vector. It then randomly and proportionally partitions instances into $k$ subsets from each stratum. While this purposive sampling method ensures the same distribution of class values in all subsets, it may not ensure the same distributions of feature values. Thus, subsampling bias problem may remain unsolved. A low subsampling bias is crucial because no matter how sophisticated a statistical machine learning algorithm is, the resulting model may be biased if its training sets fail to represent the idiosyncrasies of the complete set \cite{Braga-Neto2004,Molinaro2005}. Likewise, subsampling bias and insufficiently informative test data will inflate the \textit{EPE}s, jeopardizing the validity and reliability of a CV procedure \cite{Fu2005}. We may run the risk of choosing a biased model and drawing a wrong conclusion. 

This study follows the Ordered Systematic Subsampling (OSS) theory \cite{Lohr2009} to propose a simple yet effective method, the best-discrepancy systematic subsampling method for $k$-fold CV (hereafter the BDSCV), which ensures a smaller subsampling bias than the MCCV, stratified CV and LOO. The contribution of our proposed systematic subsampling method lies in the fact that we make a step further to reduce subsampling bias than the conventional systematic subsampling one. That is, the first subsampling instance is not chosen at random as in the conventional OSS method. Instead, it is chosen based on the Best-Discrepancy Sequence (BDS), a sequence that is a uniformly distributed sequence, finite or infinite, in which the number of any proportion of its elements falling into an arbitrary interval proportionally converts to the distance of the interval \cite{Niederreiter1992}. Our subsampling interval is also determined by the BDS. 

This article is organised as follows. We first introduce the theories of the low discrepancy sequence (LDS) and the BDS. 
We then illustrate subsampling traps in conventional $k$-fold CV. Next, we propose the best-discrepancy systematic subsampling method and the BDSCV procedure. We conduct experiments to provide empirical supports to our theories. Finally, we discuss our proposed procedure and its limitations.

\section{Methods}

\subsection{Low-discrepancy sequences and best-discrepancy sequences} 
For brevity, we restrict our discussion to the key definitions of low-discrepancy sequence (LDS) and best-discrepancy sequence (BDS), 
which are necessary to understand our procedure. We refer interested readers to the comprehensive reviews of \cite{Kuipers1974}, 
\cite{HuaWan1978} and \cite{Dick2013}. 

Let $\mathscr X$ denote a sequence $x_1,x_2,\dots,x_n,\dots$ in the unit interval $[0,1)$. The discrepancy $D_N$ of its consecutive $N$ elements is defined by 
\begin{eqnarray}\label{E1.1} 
D_N=\sup_{0\leq a<b<1} \left|\frac{C([a, b); N)}{N}-(b-a)\right|, 
\end{eqnarray}
where the counting function $C([a,b);N)$ denotes the number of element $x_n$ with $1\leq n\leq N$ for which $x_n\in [a,b)$ \cite{Kuipers1974}.   
This $\mathscr X$ is said to be an LDS if $D_N$ is low for all sufficiently large integers $N$.  
The sequence $\mathscr X$ is uniformly distributed in $[0,1)$ if $D_N\rightarrow 0$ as $N\rightarrow\infty$ \cite{Kuipers1974}. In simple terms, the sequence $\mathscr X$ is uniformly distributed if every subinterval $[a,b)$ eventually obtains its proper share in $[0,1)$. A sequence's level of uniformity is measured by its discrepancy function $D_N $. A classical theorem of Chung \cite{Chung2013} states that, in general, the PRNG does not produce a sequence with $D_N $ lower than $N^{-1/2} $. In any case, the order of the discrepancy $D_N $ cannot be lower than $N^{-1}\log N$ by \cite{Schmidt1972}. Thus, the best we can hope for $D_N$ is something like 
\begin{eqnarray}\label{E1.2} 
D_N\leq cN^{-1+\varepsilon}
\end{eqnarray}
where $\varepsilon>0$ is arbitrarily small, and $c>0$ is some constant. This explains the following definition.   

\begin{definition}\quad
A sequence $\mathscr X$ is called a BDS, if its discrepancy function $D_N$ satisfies \eqref{E1.2} 
for arbitrary $\varepsilon>0$ and some constant $c>0$.  
\end{definition}

In the following we will present some BDSs that are produced by methods from  
number theory. Of particular interests to this paper are those BDSs of the form 
\begin{eqnarray}\label{Def/na}
\{n \alpha\}, \ n=1, 2, 3, \dots
\end{eqnarray} 
where $\{\alpha\}$ denotes the fractional part of a real number $\alpha$. 
Plainly each term of \eqref{Def/na} falls into the unit interval $[0, 1)$. Denote by ${\mathscr X}(\alpha)$ the sequence \eqref{Def/na}. 

\begin{theorem}\label{Thm2.2}\quad 
The sequence ${\mathscr X}(e)$ is a BDS,  where $e$ is the base of natural logarithm. 
\end{theorem}

\begin{proof}\quad 
It follows from \cite{Baker1965} that \eqref{E1.2} holds for ${\mathscr X}(e)$ with the constant $c$ therein depending only on 
$\varepsilon$, and by \cite{Mahler1975} the dependence of $c$ on $\varepsilon$ is effectively computable. 
See also the arguments in \cite[Chap.~4, \S5]{HuaWan1978}.  
\end{proof}

\begin{remark}\quad
Of course the sequence ${\mathscr X}(\alpha)$ can be a BDS for many other $\alpha$. For example, 
it is known \cite[Chap.~4, \S5]{HuaWan1978} that if $\alpha$ is real algebraic, 
then the sequence ${\mathscr X}(\alpha)$ is a BDS. Examples of these $\alpha$ include 
$\sqrt{p}$ with $p$ ranging over all prime numbers $p$. Prior studies also propose other ways to generate BDS, 
for example the Halton sequence \cite{Halton1964} and the van der Corput sequence \cite{van der Corput1935}. 
Note that our sequence ${\mathscr X}(e)$ in Theorem~\ref{Thm2.2} belongs to none of the above examples since $e$ is transcendental. 
In this paper we have followed the suggestion of Hua and Wang \cite[Chap.~4, \S5]{HuaWan1978} to use ${\mathscr X}(e)$ rather than those previous sequences, and our numerical experiments support their claim that ${\mathscr X}(e)$ is of important practice implications. 
\end{remark}

The discrepancy properties of LDS/BDS have attracted a lot of attentions \cite{Dick2015}. Indeed, point sampling is an important research area in computational mathematics. For example, Xu and Zhou \cite{Xu2014} propose a deterministic method to produce the interpolation points and show its outstanding performance with $\mathscr L$1 minimization. Previous studies also report that the LDS can accelerate the convergence for Monte Carlo quadrature formulas \cite{Niederreiter1992} and its points are correlated to provide greater uniformity \cite{Struckmeier1995}. The LDS has been applied in many fields other than mathematics, including optimization \cite{Georgieva2010, Kucherenko2005, Pant2008, Uy2007}, machine learning \cite{Bergstra2012, Bergstra2011, Cheng2000}, computer graphics and vision \cite{Kollig2002, Keller1996, Li2003, Quinn2007, Wenzel2001}, finance \cite{Boyle1997, Paskov1995, Singhee2007, Tan2000}, and engineering \cite{Branicky2001, Dai2009, Kalagnanam1997, Lindermann2003}. Our study is inspired from these studies, although they do not directly address CV procedures. We develop the best-discrepancy systematic subsampling method and the BDSCV procedure based on the BDS. 

\subsection{Numerical remarks on Section~2.1}  

Panel A-C in Figure 1 shows the scatter plots of three random sequences, each of which contains 500 elements. The three sequences contain clear clumps and voids so that it is highly possible that the distance between two consecutive random elements may be either very small or very large. To formally compute the $D_N $ of each sequence, we randomly choose 50 pairs of interval $[a,b) $ and compute the discrepancy to uniform distribution for these two sequences using \eqref{E1.1} respectively. Results are shown in Figure 2, from which we find that the average $D_N$ and the variance of 50 measures of each random sequence are quite large. That is to say, these three sequences are not quite uniformly distributed. 

\medskip 

\centerline{Figure 1 Here.}  

\medskip 

\centerline{Figure 2 Here.}  

\medskip 

Panel D in Figure 1 illustrates a BDS of 500 elements. The interval between two consecutive elements of a BDS is determined by the property of uniformity so that the elements scatter evenly to avoid clumps and voids. 
We formally compute the $D_N $ of BDS by randomly choosing 50 pairs of interval $[a,b)$ using \eqref{E1.1}.  
Figure 2 shows that the resulting discrepancy of BDS is close to zero and the variance of the 50 measures of BDS is greatly smaller than that of its random counterpart. That is, unlike the elements of a random number sequence, the elements of a BDS are as uniformly distributed as possible. Successive elements are inserted in a position as far as possible from the other elements in order to avoid clustering and voids. The elements generated sequentially fill in the larger gaps between the previous elements of the sequence \cite{Cheng2000}. 

Random sequences cause subsampling bias in $k$-fold CV. To illustrate subsampling traps, we generate an artificial dataset of 16 instances with four binary features and one label vector with 16 classes (i.e. each combination of the values of the four binary features corresponds to one class). We duplicate this dataset five times and form a complete set of 80 instances. A logistic regression model is used to map out the relations between features and labels. This should be an easy task, as the classification boundary of this dataset are perfectly clear. Therefore, if the 80 instances are well partitioned into representative subsets, then the error rate of both Holdout and CV procedure should be zero. 

We conduct $k$-fold CV with the logistic regression model, where $k$ ranges from two to 10. Each $k$-fold CV procedure is repeated 50 times with different random seeds to circumvent the risk of obtaining largely dissimilar subsets. Results are shown in Figure 3. The error rate of the $k$-fold CV is surprisingly high and the error rate drops when the number of $k$ increase. However, the error rates are always larger than zero. The box plots also suggest that there are large variations over 50 repetitions. This illustration vividly demonstrates that there are actually traps in conventional $k$-fold CV. A major threat to the validity and reliability of $k$-fold CV lies in subsampling bias by the clumps and voids illustrated in Figure 1, which largely jeopardize the representativeness of subsets. Results of model evaluation and selection could be problematic. 

\medskip 

\centerline{Figure 3 Here.}  

\medskip 

\subsection{The Best-Discrepancy Systematic Subsampling Method} 
We propose a systematic subsampling method based on the uniformity property of the BDS. The theoretical foundation of our proposed method stems from sampling theory, which proves that the ordered systematic subsampling (OSS) method can obtain smaller sampling error than the SRS and the stratified sampling \cite{Lohr2009}. The OSS method works as follows. A complete dataset should be firstly sorted either ascendingly or descendingly so that adjacent elements tend to be more similar than elements that are farther apart. Then, a fixed subsampling interval (also called a ``skip") $k$ is defined as $k=n/m$, where $n$ is the dataset size and $m$ is the subset size. The subsampling starts by selecting an element from the ordered complete dataset at random and then every $k^{\rm th}$ element is selected until a subset of $m$ elements is formed. The ordered systematic subsampling outperforms the SRS because the former can avoid choosing two elements adjacent to each other and the resulting subsets are forced to compose of various values, meaning that the discrepancy between subsets caused by subsampling bias can be reduced \cite{Lohr2009}.

We improve the OSS by replacing the fixed subsampling interval with a BDS.  Suppose $D=(d_1,\dots,d_n)$ is a one-dimensional dataset with $n$ instances. We use this vector notation to emphasise the importance of the order of $D$. Let ${\widetilde D}=({\widetilde d}_1,\dots,{\widetilde d}_n)$ be the sorted $D$, meaning that $\widetilde D$ is a permutation of the coordinates of $D$, so that ${\widetilde d}_1<\dots<{\widetilde d}_n$. This $\widetilde D$ is to be divided into $k$ subsets, each of which contains $m$ instance (i.e. $n=km$). To do so, an integer sequence derived from a BDS is needed that refers to the subsampling frame of a dataset. 

Let $X=(x_1,\dots,x_n)$ be the vector of the first $n$ elements of ${\mathscr X}(e)$. Let $\widetilde X=({\widetilde x}_1,\dots,{\widetilde x}_n)$ be the sorted $X$, meaning that $\widetilde X$ is a permutation of the coordinates of $X$, so that ${\widetilde x}_1<\dots<{\widetilde x}_n$, The ranking vector $R=(r_1,\dots,r_n)$ of $X$ is such that ${\widetilde x}_{r_j}=x_j$ for $j=1,\dots,n$. Then, given $n=km$, we can chunk $R$ into $k$ disjoint subsets of the same size $m$, that is 
\begin{eqnarray}\label{Def/R=}
R=\bigcup_{1\leq j\leq k}R_j
\end{eqnarray}
where $R_1=(r_1, \ldots, r_m), R_2=(r_{m+1}, \ldots, r_{2m})$, etc. 
This $R$ inherits the property of uniformity. We can use $R$ to partition $\widetilde D$ into $k$ subsets $\widetilde D=\cup_{1\leq j\leq k}{\widetilde D}_j$, where for $j=1,\dots,k$,
\begin{eqnarray}\label{E2.4}
{\widetilde D}_j=({\widetilde d}_r: r\in R_j). 
\end{eqnarray}
Note that each subset ${\widetilde D}_j$ is written as a vector for the same reason.    

Our proposed best-discrepancy systematic subsampling method can make sure the instances in a subset as heterogeneous as possible while different subsets as homogenous as possible. According to the theory of analysis of variance \cite{Lohr2009}, we know that the \textit{Sum of Squares of Total} (\textit{SSTO}) is decomposed into the \textit{Sum of Squares Within subsets} (\textit{SSW}) plus the \textit{Sum of Squares Between subsets}(\textit{SSB}). The level of heterogeneity of instances within subsets can be measured by the \textit{Intra-class Correlation Coefficient} (\textit{ICC}) \cite{Lohr2009},
\begin{eqnarray}\label{E2.5}
ICC=1-\frac{km}{km-1}\frac{SSW}{SSW+SSB},
\end{eqnarray}
where
\begin{eqnarray}\label{E2.6}
SSW=\sum_{i=1}^{k}\sum_{j=1}^{m}(d_{ij}-{\overline d}_i),
\end{eqnarray}
\begin{eqnarray}\label{E2.7}
SSB=n\sum_{i=1}^{k}({\overline d}_i-\overline d),
\end{eqnarray}
\begin{eqnarray}\label{E2.8}
SSTO=SSW+SSB.
\end{eqnarray}
And Equation \eqref{E2.5} can be transformed as 
\begin{eqnarray}\label{E2.9}
SSW=\frac{(1-ICC)(n-1)SSTO}n.
\end{eqnarray}

Equation \eqref{E2.9} implies that \textit{SSW} becomes larger when \textit{ICC} becomes negative (i.e. the instances within a subset are dispersed more than a randomly chosen subset would be). Prior study \cite{Lohr2009} proves that  to have a negative \textit{ICC}, the dataset should be ordered, so that it exhibits a positive autocorrelation (i.e. subsequent instances have similar or the same values) and instances should be chosen with a skip from the ordered dataset so that a subset is forced to compose of both small values and large values. The best-discrepancy systematic subsampling method follows exactly these guidelines to allocate instances. In addition, it is important to mention that the subsampling interval is naturally and uniformly set by the uniformity property of a BDS. The subsampling interval of the BDS guarantees that a subset is composed of heterogeneous instances. 

\subsection{Numerical remarks on Section~2.3}  
We illustrate the relations among the BDS, \textit{ICC} and \textit{SSW} in the following example. Suppose a one-dimensional dataset $D$ is given, and the 21 instances of $D$ are to be partitioned into three subsets with \eqref{E2.4}. We use ${\mathscr X}(e)$ to generate a tiny BDS using $e$'s 38 digits after the decimal place (see the first column of Table 1). The property of uniformity of the BDS ensures that the distribution of values in a subset should be the same or close to that of the complete set, and the between-subset variance should equal or be close to zero. For example, the second column of Table 1 demonstrates that the variances of the three subsets $(\sigma _1^{2}=\sigma _2^{2}=\sigma _3^{2}=0.08)$ are all close to that of the complete set $\sigma _X^{2}=0.08$, provided that the complete set is evenly chunked into three subsets using the BDS. The third and fourth columns of Table 1 show that the ranking vector $R$ of the BDS exhibits the same uniformity property as the means of the three subsets are quite close (10, 10, and 12 respectively). 

Then, $D$ is ascendingly sorted as $\widetilde D$ so that adjacent instances tend to be more similar than instances that are farther apart. The 21 instances are partitioned into three subsets with \eqref{E2.4} (see the fifth to seventh columns of Table 1). The three subset variances are quite close to that of $D$ $(\sigma _D^{2}=0.01)$ and the \textit{SSB} is almost zero (0.0082). As expected, the \textit{ICC} is negative ($-0.1242$) and the \textit{SSW} is 0.2168. The instances with adjunct or the same values (for example, the $9^{\rm th}, 10^{\rm th}$ and $11^{\rm th}$ elements of $D$) are forced to spread out between different subsets. The ranking vector $R$ makes the distributions of values highly similar among three subsets, as none of them consists of unrepresentative (i.e. mainly low or high) values. 

To compare, we use the SRS to split $D$ into three subsets and repeat the procedure 50 times. The averages of \textit{SSB} of these SRS procedures is 0.0208, which is 2.531 times larger than that of the best-discrepancy systematic subsampling method. This experiment shows that subsampling bias can be greatly reduced by the proposed method, even though the sample size of this tiny dataset is only 21. We conduct an additional experiment with 500 datasets of different sizes (ranging from 54 to 999) and with different levels of variance (ranging from 0 to 0.025), skewness (ranging from 0 to 0.116), and kurtosis (ranging from 0.506 to 31.398). The results suggest that regardless the characteristics of a dataset, the \textit{SSB} of the best-discrepancy systematic subsampling method is much smaller (23.16\%) than that of the SRS. 

\medskip 

\centerline{Table 1 Here.}  

\medskip

 In short, this example proves that the best-discrepancy systematic subsampling method can satisfy Lohr's \cite{Lohr2009} two conditions that make the \textit{ICC} negative. In contrast, the SRS cannot guarantee all subsets are of heterogeneous instances. Indeed, there are too many possibilities to randomly generate $k$ subsets. It is then difficult to avoid ``bad'', quasi-homogeneous subsets with only or mainly small or large values. Therefore, the best-discrepancy systematic subsampling procedure outperforms the SRS in terms of subset representativeness.

\newpage 
\section{The BDSCV Procedure}We propose a $k$-fold CV procedure that is built upon the best-discrepancy systematic subsampling procedure. The pseudocode of BDSCV is shown in Figure 4. 

\medskip 

\centerline{Figure 4 Here.}  

\medskip

Suppose we are given an independently and identically distributed $m$ dimensional dataset $D$ of $n$ instances, where $D\in\mathfrak {\Bbb R}^{n\times m}$. Each instance in $D$ consists of a class label vector $Y$ and vectors of measured feature covariates $X$. As discussed in the previous section, $D$ is sorted before subsampling. However, it is difficult to decide along which axis to sort $D$. An arbitrary yet effective way is to transform $D$ from multi-dimensional space to one-dimensional space using a dimensionality-reduction technique $H$, which produces a linear transformation $T\in\mathfrak {\Bbb R}^{m\times1}$ and results in $E=DT$, where $E\in\mathfrak {\Bbb R}^{n\times1}$ and $E$ is the underlying structure of $D$ and points to a direction where there is the most variance and where the elements in $E$ are most spread out. 

Simply stated, it is just like to project all dimensions of a dataset on the component vector $E$, which is a combination of all dimensions with weights. The instances with close or the same values in $E$ are considered similar in the terms of all dimensions but not just of label classes. And $E$ can be served as an axis along which $D$ is sorted so that instances that are similar to some extent are then gathered and separated from dissimilar ones. Therefore, we ascendingly sort $E$ as $\widetilde E$ and then sort $D$ as $\widetilde D$, so that $\widetilde E=\widetilde DT$. These $\widetilde E$ and $T$ are then discarded and will not be included in subsequent analyses. 

Next, we generate a BDS with $n$ elements and then a ranking vector $R$ accordingly. We partition $\widetilde D$ into $k$ subsets, using the same allocation mechanism in \eqref{E2.4}. 

Thirdly, the $k$ subsets of $\widetilde D$ are used to conduct a standard $k$-fold CV procedure with a classifier $A$. Let $F_k$ be the $k$-th fold that serves as the test set, and ${\widetilde D}_k$ be the training set obtained by removing the instances in $F_k$ from $\widetilde D$. The estimated \textit{EPE} over $k$ folds with the dimensionality-reduction technique $H$ is estimated as in \cite{Molinaro2005}: 
\begin{eqnarray}\label{E2.10}
EPE(\widetilde D)_{H,A}=\frac1K\sum_{k=1}^{K}\widehat{EPE(F_k)}. 
\end{eqnarray}
And the variance of the CV estimator is estimated as in \cite{Molinaro2005}: 
\begin{eqnarray}\label{E2.11}
\sigma_{H,A}^{2}=\frac1K \sum_{k=1}^{K}(EPE(F_k)_{H,A}-EPE(\widetilde D)_{H,A})^{2}. 
\end{eqnarray}

Finally, different dimensionality-reduction techniques may lose different amounts of information from the complete dataset \cite{Cunningham2015}. Sweeping different candidate techniques, such as Principal Component Analysis (PCA), Factor Analysis (FA), kernel PCA (kPCA), or truncated Singular Value Decomposition (tSVD), will enable us to fine tune the BDSCV procedure, so that the original space of $D$ can be reduced to one-dimensional space while retaining the most important information as much as possible in $E$. Therefore, previous steps are repeated using different dimensionality-reduction techniques in order to find the best $H$ in terms of $EPE(\widetilde D)_{H,A}$ and its corresponding $\sigma _{H,A}^{2}$ as the final CV results for the classifier $A$. 

\section{Experiments}
We first test the logistic regression model in Section 2.2 with the proposed 10-fold BDSCV. The model achieves zero error rate in every fold, which suggests that the BDSCV can eliminate subsampling bias to a great extent. We then test the BDSCV with the Penn Machine Learning Benchmark (PMLB) datasets \cite{Olson2017}. Note that we eliminate 10 PMLB datasets that are too big to be computed on a personal computer. The characteristics of these 156 datasets are summarized in Table 2. We compare the performance of our proposed BDSCV to MCCV, stratified MCCV and LOO.

\medskip 

\centerline{Table 2 Here.}  

\medskip 

The MCCV and stratified MCCV procedures are conducted using the \textit{scikit-learn} cross-validation module (version 0.19.2) with 50 repetitions for the three classifiers respectively. Our proposed BDSCV is mainly built on the same \textit{scikit-learn} module except the subsampling part. We use the MCCV as the benchmark and compute the following three comparative performance metrics as: 

(a) the comparative \textit{EPE} ratio of a classifier $A$, which is 
\begin{eqnarray}\label{E4.1}
\varphi EPE_A=\frac{EPE_{BDSCV/Stratified/LOO_A}}{EPE_{MCCV_A}}; 
\end{eqnarray}
(b) the comparative variance ratio of a classifier $A$, which is 
\begin{eqnarray}\label{E4.2}
\varphi\sigma _A^{2}=\frac{\sigma _{BDSCV/Stratified/LOO_A}^{2}}{\sigma _{MCCV_A}^{2}};
\end{eqnarray}
and (c) the comparative computational burden ratio of the $k$-fold CV procedures, which is 
\begin{eqnarray}\label{E4.3}
\varphi Time=\frac{Time_{BDSCV/Stratified/LOO}}{Time_{MCCV}}. 
\end{eqnarray}

Our experiments are conducted on a desktop personal computer (CPU: Intel Core i7-7700 eight cores, Memory: 64G, and SSD: 256GB). Note that computing time may slightly vary depending on the overall load and specifications of a computer. Table 3 shows that the number of datasets in which each of the four CV procedures achieve the same lowest \textit{EPE}s and variances (it is possible that multiple CV procedures obtain the lowest performance metrices in a dataset). As expected, the BDSCV achieves the lowest \textit{EPE}s and variances in most datasets, followed by the LOO and the stratified MCCV. Table 3 suggests that there are significant subsampling biases among different CV procedures and if we count on the most widely used 10-fold MCCV to do model selections, then we probably make the wrong choices.  

\medskip 

\centerline{Table 3 Here.}  

\medskip 

We summarize the three comparative performance metrics in Table 4. The \textit{EPE}s, variances and computational time of each classifier and of each dataset can be found in Supplementary Material. 

The comparisons between the BDSCV and the MCCV are as follows (see Panel A in Table 4). On average, the $\textit{EPE}\textit{\ensuremath{_{Logistic}}}$ and $\sigma _{Logistic}^{2} $ of the former is 94.04\% and 70.30\% of that of the latter respectively; the $\textit{EPE}\ensuremath{_{\textit{DecisionTree}}}$ and $\sigma _{DecisionTree}^{2} $ of the BDSCV is 88.02\% and 77.50\% of that of the MCCV respectively; and the $\textit{EPE}\ensuremath{_{\textit{Na{\"{\i}}ve Bayes}}}$ and $\sigma_{\textit{Na{\"{\i}}ve Bayes}}^{2} $ of the former is 96.41\% and 72.00\% of that of the latter respectively. On average, the BDSCV can lower the \textit{EPE}s around 7.18\% and the variances around 26.73\% regardless classifiers. In comparison, the stratified MCCV can slightly reduce subsampling bias by decreasing the \textit{EPE}s of the MCCV around 1.58\% and the variances around 11.85\%. The LOO can lower the \textit{EPE}s around 2.50\% but its variances are much higher than the any other CV procedure (see Panel B and C in Table 4). The computational time of the BDSCV is just 8.64\% of the MCCV, 8.67\% of the stratified MCCV and 16.72\% of the LOO (see Panel D in Table 4). These results suggest that the proposed BDSCV outperforms the other CV procedures as it can extrude significant subsampling bias while using a small fraction of computational time of the other CV procedures. Being computationally efficient is especially important where it is expensive or impractical to train multiple models. 

Finally, our proposed CV procedure may not be a cure-all and be more pertinent to a dataset with particular characteristics. We compute the aspect ratio (i.e. the number of features divided by the number of instances) of the 156 datasets and divide them into four groups based on the quantiles of the aspect ratio. The first group includes 67 datasets whose number of features is smaller or equal to 1\% of number of instances. The second group includes 20 datasets whose percentages of number of features to number of instances are in the range of $(1\%,2\%]$. The third group includes 33 datasets whose aspect ratios are in the range of $(2\%,5\%]$. And the fourth group include 36 datasets whose aspect ratios are larger than 5\%. Table 4 shows that the BDSCV can significantly reduce subsampling bias in the third and fourth groups of datasets, as the \textit{EPE}s of the BDSCV are significantly smaller. These findings imply that the BDSCV could be particularly effective in analysing in bioscience datasets, which usually are characterized as high dimensionality and relatively small sample sizes \cite{Daz-Uriarte2006, Fu2005}.

In short, these experiments confirm that the MCCV's suffers a lot from subsampling bias. The BDSCV can reduce subsampling bias in the \textit{EPE}s and variances of the three classifiers to a significant extent. The improvements of the stratified MCCV are not as substantial as the BDSCV. The LOO does not have significant advantages compared to the other CV procedures. Both the stratified MCCV and the LOO are not as computationally efficient as the BDSCV. In addition, the decision tree is the most sensitive classifier to subsampling bias in terms of \textit{EPE}s, especially in the datasets with high aspect ratios (i.e. the $3^{\rm rd}$ and $4^{\rm th}$ groups) in which the BDSCV can deflate the $\textit{EPE}\ensuremath{_{\textit{DecisionTree}}}$ around 15\% on average. The logistic regression is the most sensitive classifier to subsampling bias in terms of variances, especially in the datasets with relatively low aspect ratio (i.e. the $1^{\rm st}, 2^{\rm nd}$ and $3^{\rm rd}$ groups) in which the BDSCV can reduce variances around 34\%.

\section{Discussion and Closing Remarks}
A statistical machine learning model could be affected by the idiosyncrasies of training data if the discrepancy between $k$ subsets is too high. As a result, the \textit{EPE}s and variances of a MCCV procedure are contaminated by subsampling bias and cannot reflect the true generalization performance of a model's inductive mechanism. The inflated metrics will interfere with data scientists making right decisions. Our experimental results confirm that subsampling bias in the MCCV procedure greatly jeopardizes its validity and reliability for model selection even after 50 repetitions. Our experiments also confirm that the subsampling bias can be reduced through Lohr's guidelines \cite{Lohr2009}{\textemdash} systematic subsampling from an ordered population with a skip. The BDSCV follows these guidelines so that it is relatively more precise, stable and computationally efficient than all the other three CV procedures. 

It is important to point out that the BDSCV is not meant to completely replace other CV procedures. As the results of our experiments suggest that the advantages of the different types of CV may depend on dataset characteristics and decision boundary. We recommend using the BDSCV in the datasets characterized with high aspect ratios (i.e. small dataset size and high dimensionality) and for the non-parametric classifier such as the decision tree. Besides, the BDSCV is much more computational efficient than the other CV procedures so it is suitable for the time-critical statistical machine learning problems. Finally, it is beneficial to estimate the true generalization performance of a learning function using different CV procedures in order to choose the model with the lowest \textit{EPE}s and variances. Further studies are also necessary to determinate what is the most appropriate CV procedure for a particular dataset and for a statistical machine learning model. 

In conclusion, this study sheds more light on the literature of CV and statistical machine learning model evaluation by proposing the BDSCV procedure, which is a new way to estimate a model's relatively objective generalizability by significantly reducing subsampling bias in existing CV procedures. It also helps the model to pick out the hidden underlying mapping between input-output vectors and to remain stable when the training set changes, leading to more precise estimations and more confident inferences. What is more, our proposed procedure is straightforward and computationally efficient. Unlike other information-preserving subsampling techniques \cite{Groot1999} that need to understand the complete dataset in great detail, our procedure does not make the CV computation more complex. Lastly, our subsampling method may replace the random mechanism used in many statistical machine learning algorithms. 

The Python source codes, Supplementary Material, and the 156 benchmark datasets are available at http://cssci.wh.sdu.edu.cn/bdscv

\Acknowledgements{
The authors would like to thank the referees for helpful suggestions on an earlier draft of the paper. 
The second author is supported by the National Science Foundation of China under Grant 11531008, the Ministry of Education of China under Grant IRT16R43, and the Taishan Scholar Project of Shandong Province. }


\newpage
\begin{figure}[h]
	\centering
	\includegraphics[width=0.7\linewidth]{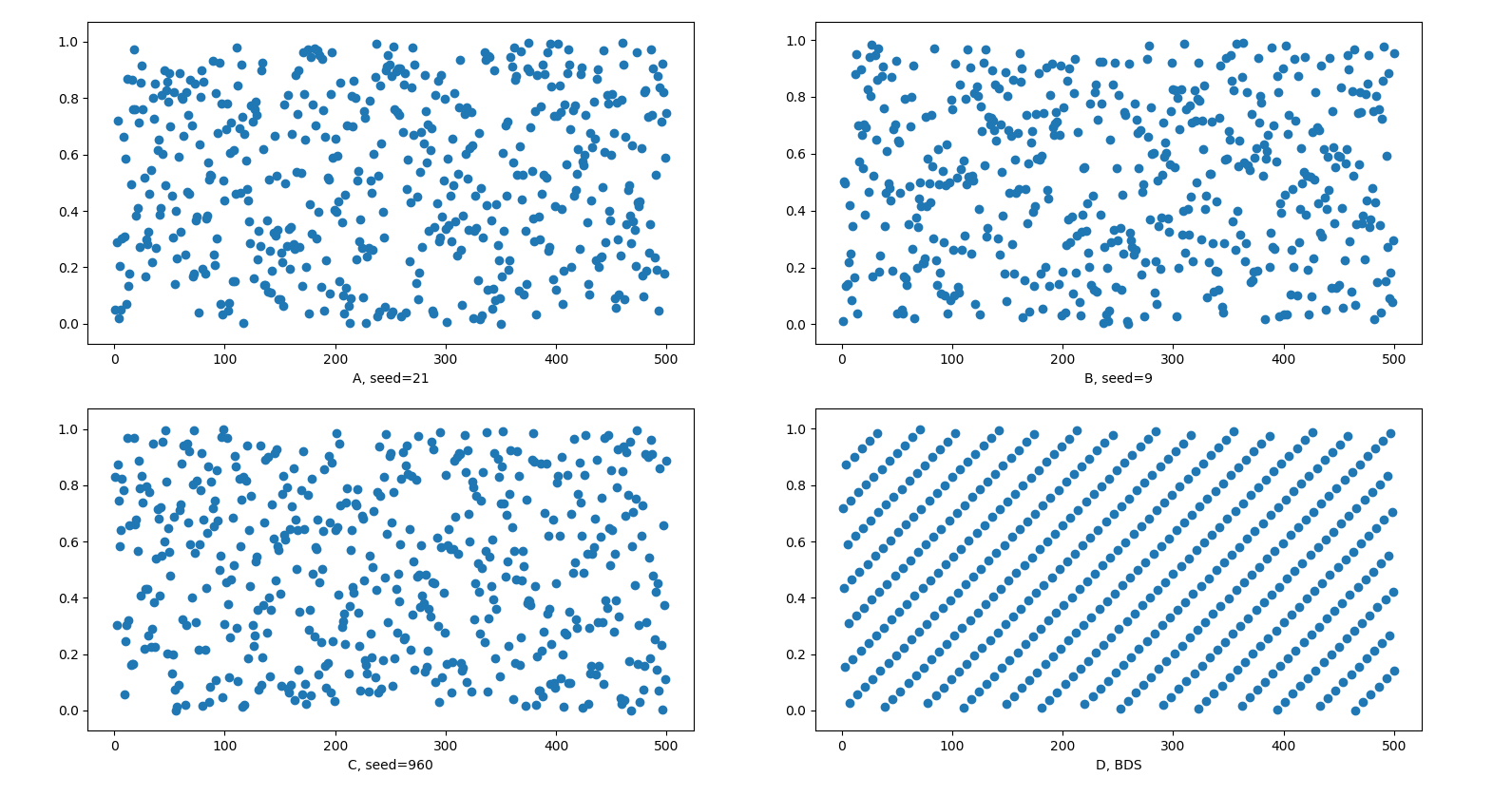}
	\caption{The Scatter Plots of Random Sequences and of BDS}
	\label{fig:Figure1}
\end{figure}

\begin{figure}[h]
	\centering
	\includegraphics[width=0.7\linewidth]{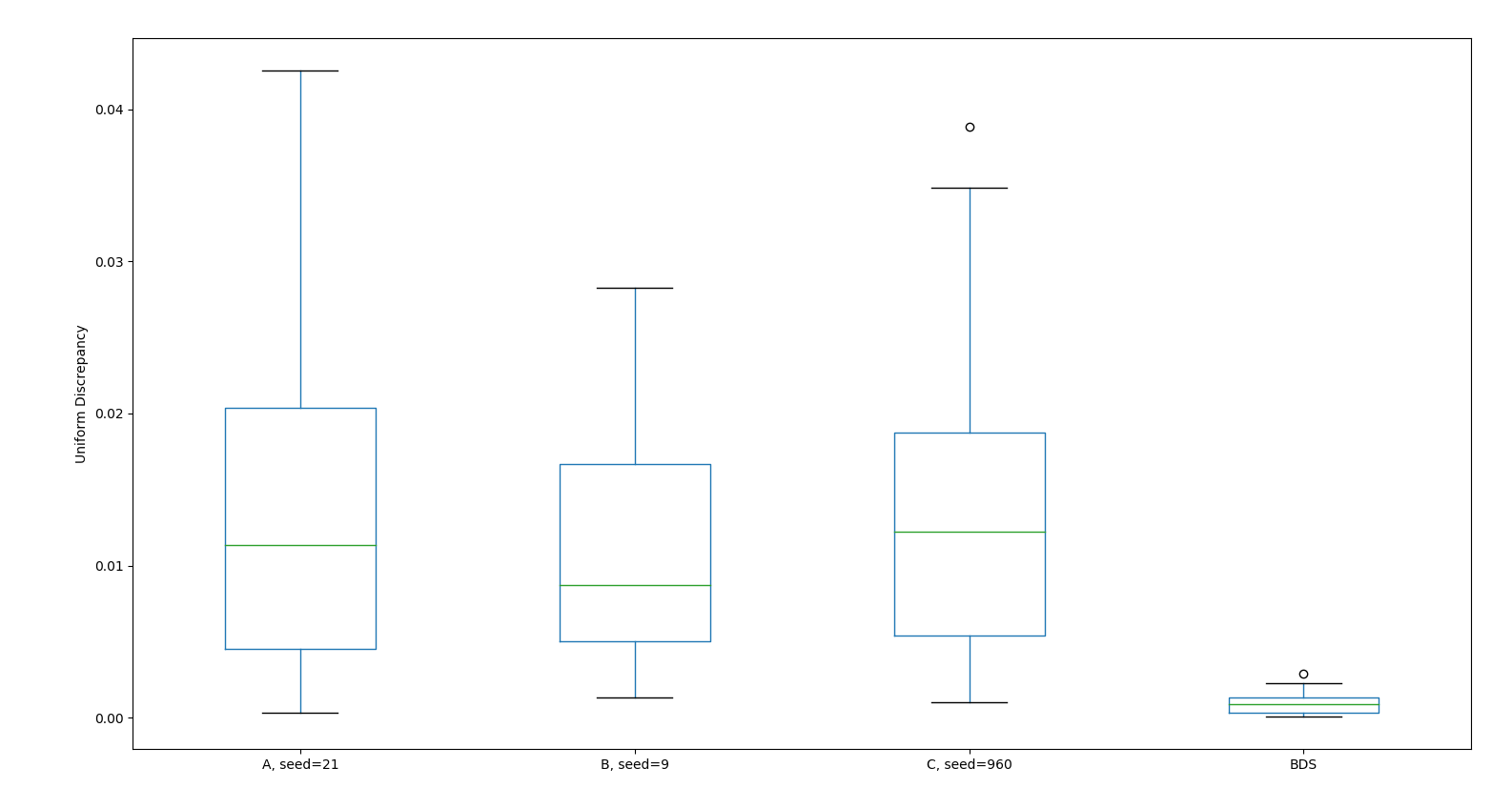}
	\caption{The Uniform Discrepancy of Random Sequences and of BDS}
	\label{fig:Figure2}
\end{figure}

\begin{figure}[h]
	\centering
	\includegraphics[width=0.7\linewidth]{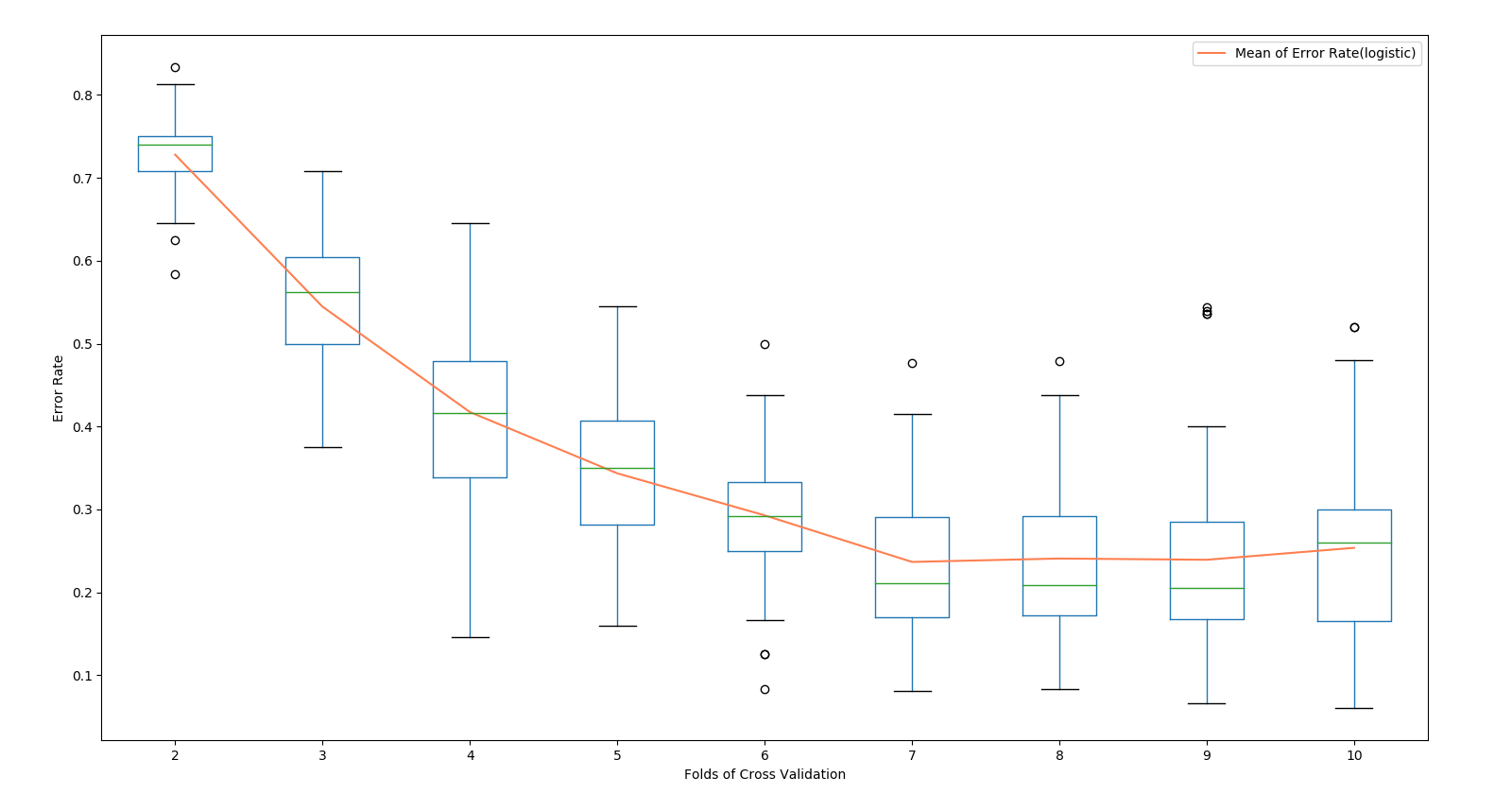}
	\caption{Error Rate in Conventional k-fold CV}
	\label{fig:Figure3}
\end{figure}

\begin{figure}[h]
	\centering
	\includegraphics[width=0.7\linewidth]{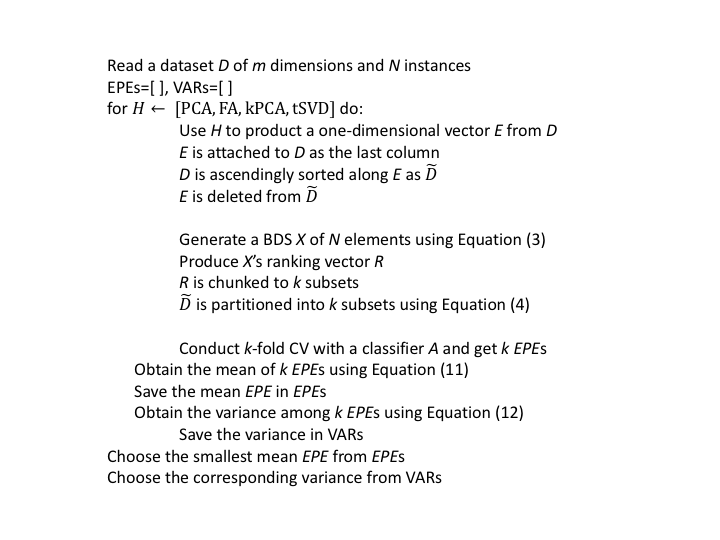}
	\caption{Pseudocode of the BDSCV Procedure}
	\label{fig:Figure4}
\end{figure}

\begin{table}[h]
	\centering
	\setlength{\abovecaptionskip}{0pt}
	\setlength{\belowcaptionskip}{10pt}
	\caption{Elements with the Best Discrepancy Systematic Subsampling Method}
	\begin{tabular}{ccccccc}
		\hline
		\multicolumn{2}{c}{BDS} & \multicolumn{2}{c}{R} & \multicolumn{3}{c}{ $\widetilde{D}$}               \\ \hline
		$X_{I}$    & $subset_{X_{I}}$ & $R_{I}$ &  $subset_{R_{I}}$ &  $\widetilde{D}_I$  &  $\widetilde{d}_i$ after partition & $subset_{d_{i}}$  \\ \hline
		0.7183     &            & 16      &             & 1.50$_{1}$  & 1.70$_{16}$            &        \\
		0.4366     &            & 10      &             & 1.53$_{2}$  & 1.67$_{10}$            &        \\
		0.1548     & I          & 4       & I           & 1.55$_{3}$  & 1.60$_{4}$             & I      \\
		0.8731     & $\mu_{1}$:\ 0.44       & 19      & $\mu_{1}$:\ 10          & 1.60$_{4}$  & 1.80$_{19}$            & $\mu_{1}$:\ 1.66   \\
		0.5914     & $\sigma_{1}^{2}$:\ 0.08       & 13      & $\sigma_{1}^{2}$:\ 36          & 1.62$_{5}$  & 1.68$_{13}$            & $\sigma_{1}^{2}$:\ 0.01   \\
		0.3097     &            & 7       &             & 1.63$_{6}$  & 1.64$_{7}$             &        \\
		0.028      &            & 1       &             & 1.64$_{7}$  & 1.50$_{1}$             &        \\ \hline
		0.7463     &            & 17      &             & 1.65$_{8}$  & 1.76$_{17}$            &        \\
		0.4645     &            & 11      &             & 1.66$_{9}$  & 1.67$_{11}$            &        \\
		0.1828     & II         & 5       & II          & 1.67$_{10}$ & 1.62$_{5}$             & II     \\
		0.9011     & $\mu_{2}$:\ 0.47       & 20      & $\mu_{2}$:\ 10          & 1.67$_{11}$ & 1.90$_{20}$             & $\mu_{2}$:\ 1.69   \\
		0.6194     & $\sigma_{2}^{2}$:\ 0.08       & 14      & $\sigma_{2}^{2}$:\ 36          & 1.68$_{12}$ & 1.69$_{14}$            & $\sigma_{2}^{2}$:\ 0.01   \\
		0.3377     &            & 8       &             & 1.68$_{13}$ & 1.65$_{8}$             &        \\
		0.0559     &            & 2       &             & 1.69$_{14}$ & 1.53$_{2}$             &        \\ \hline
		0.7742     &            & 18      &             & 1.70$_{15}$ & 1.78$_{18}$            &        \\
		0.4925     &            & 12      &             & 1.70$_{16}$ & 1.68$_{12}$            &        \\
		0.2108     & III        & 6       & III         & 1.76$_{17}$ & 1.63$_{6}$             & III    \\
		0.9291     & $\mu_{3}$:\ 0.50      & 21      & $\mu_{3}$:\ 12          & 1.78$_{18}$ & 1.92$_{21}$            & $\mu_{3}$:\ 1.70   \\
		0.6474     & $\sigma_{3}^{2}$:\ 0.08       & 15      & $\sigma_{3}^{2}$:\ 36          & 1.80$_{19}$ & 1.70$_{15}$            & $\sigma_{3}^{2}$:\ 0.01   \\
		0.3656     &            & 9       &             & 1.90$_{20}$ & 1.66$_{9}$             &        \\
		0.0839     &            & 3       &             & 1.92$_{21}$ & 1.55$_{3}$             &        \\ \hline
		SSW        & 1.6678     &         & 756         &        & 0.2168            &        \\
		SSB        & 0.0109     &         & 14          &        & 0.0082            &        \\
		ICC        & -0.1591    &         & -0.1455     &        & -0.1241           &        \\ \hline
	\end{tabular}
\end{table}

\begin{table}[h]
\centering
\setlength{\abovecaptionskip}{0pt}
\setlength{\belowcaptionskip}{10pt}
\caption{Descriptive Statistics of 156 PMLB Datasets}
\begin{tabular}{ccccc}
	\hline
	& \#Instances & \#Features & \#Labels & Aspect Ratio \\ \hline
	Minimum & 32          & 2          & 2        & 0            \\
	Maximum & 28056       & 240        & 26       & 0.54         \\
	Mean    & 2056.34     & 24.77      & 3.83     & 0.05         \\
	Median  & 2000        & 9          & 2        & 0.01         \\ \hline
\end{tabular}
\end{table}

\begin{table}[h]
\centering
\setlength{\abovecaptionskip}{0pt}
\setlength{\belowcaptionskip}{10pt}
\caption{Number of Datasets that Each CV Reaches the Lowest EPEs and Variances}
\begin{threeparttable}
\begin{tabular}{ccccccccccccc}
	\hline
	\multirow{2}{*}{Classifier} &  & \multicolumn{2}{c}{\textit{BDSCV}} & \textit{} & \multicolumn{2}{c}{\textit{MCCV}} & \textit{} & \multicolumn{2}{c}{\textit{Stratified MCCV}} & \textit{} & \multicolumn{2}{c}{\textit{LOO}} \\ \cline{3-4} \cline{6-7} \cline{9-10} \cline{12-13}
	&  & \textit{EPE}     & \textit{Var}    & \textit{} & \textit{EPE}    & \textit{Var}    & \textit{} & \textit{EPE}          & \textit{Var}         & \textit{} & \textit{EPE}    & \textit{Var}   \\ \hline
	$L^{1}$                          &  & 118              & 101             &           & 8               & 15              &           & 14                    & 43                   &           & 23              & 2              \\
	$DT^{2}$                         &  & 115              & 102             &           & 8               & 21              &           & 12                    & 43                   &           & 43              & 12             \\
	$NB^{3}$                         &  & 125              & 94              &           & 4               & 8               &           & 8                     & 55                   &           & 22              & 2              \\ \hline
\end{tabular}
\begin{tablenotes}
	\footnotesize
	\item[1] L: Logistic
	\item[2] DT: Decision Tree
	\item[3] NB: Native Bayes
\end{tablenotes}
\end{threeparttable}
\end{table}

\begin{table}[h]
\centering
\setlength{\abovecaptionskip}{0pt}
\setlength{\belowcaptionskip}{10pt}
\caption{Comparisons between Different CV Procedures}
\begin{threeparttable}
\begin{tabular}{ccccccc}
	\hline
	\multicolumn{7}{c}{\textbf{A:  BDSCV VS MCCV}}                                           \\ \hline
	Aspect Ratio        & $EPE_{L}$         & $\sigma_{L}^{2}$           & $EPE_{DT}$         & $\sigma_{DT}^{2}$            & $EPE_{NB}$          & $\sigma_{NB}^{2}$            \\
	Q1      & 97.95\%  & 64.29\%    & 90.51\%  & 87.88\%    & 98.61\%  & 75.11\%    \\
	Q2      & 94.45\%  & 67.06\%    & 90.25\%  & 64.22\%    & 96.57\%  & 57.31\%    \\
	Q3      & 91.37\%  & 66.80\%    & 82.92\%  & 78.72\%    & 95.89\%  & 72.47\%    \\
	Q4      & 92.36\%  & 83.06\%    & 88.41\%  & 79.17\%    & 94.59\%  & 83.11\%    \\
	Average & 94.04\%  & 70.30\%    & 88.02\%  & 77.50\%    & 96.41\%  & 72.00\%    \\ \hline
	\multicolumn{7}{c}{\textbf{B:  Stratified MCCV VS MCCV}}                                                           \\ \hline
	Aspect Ratio        & $EPE_{L}$         & $\sigma_{L}^{2}$           & $EPE_{DT}$         & $\sigma_{DT}^{2}$            & $EPE_{NB}$          & $\sigma_{NB}^{2}$            \\
	Q1      & 99.48\%  & 72.02\%    & 98.91\%  & 91.23\%    & 99.59\%  & 80.99\%    \\
	Q2      & 97.96\%  & 83.79\%    & 99.53\%  & 93.58\%    & 99.91\%  & 89.41\%    \\
	Q3      & 97.25\%  & 90.41\%    & 94.35\%  & 93.08\%    & 98.79\%  & 85.73\%    \\
	Q4      & 98.40\%  & 94.25\%    & 97.62\%  & 93.73\%    & 99.25\%  & 85.59\%    \\
	Average & 98.28\%  & 85.12\%    & 97.60\%  & 92.90\%    & 99.39\%  & 86.43\%    \\ \hline
	\multicolumn{7}{c}{\textbf{C:  LOO VS MCCV}}                                                            \\ \hline
	Aspect Ratio        & $EPE_{L}$         & $\sigma_{L}^{2}$           & $EPE_{DT}$         & $\sigma_{DT}^{2}$            & $EPE_{NB}$          & $\sigma_{NB}^{2}$            \\
	Q1      & 99.72\%  & 44036.96\% & 93.16\%  & 36397.24\% & 100.01\% & 37667.87\% \\
	Q2      & 100.20\% & 7897.06\%  & 92.52\%  & 7693.12\%  & 100.36\% & 7000.07\%  \\
	Q3      & 97.43\%  & 4977.84\%  & 91.25\%  & 4084.19\%  & 101.24\% & 6016.48\%  \\
	Q4      & 98.25\%  & 4387.51\%  & 98.45\%  & 4411.73\%  & 97.44\%  & 3858.06\%  \\
	Average & 98.90\%  & 15324.84\% & 93.85\%  & 13146.57\% & 99.76\%  & 13635.62\% \\ \hline
	\multicolumn{7}{c}{\textbf{D:  Computational Time Comparisons}}                                                            \\ \hline
	Aspect Ratio        & $\frac{BDSCV}{MCCV}$         &  $\frac{Stratified}{MCCV}$          &$\frac{LOO}{MCCV}$        & $\frac{BDSCV}{Stratified}$  &  $\frac{LOO}{Stratified}$        & $\frac{BDSCV}{LOO}$             \\
	Q1      & 9.15\%   & 99.62\%    & 886.01\% & 9.19\%     & 889.44\% & 3.08\%     \\
	Q2      & 8.48\%   & 99.68\%    & 152.33\% & 8.51\%     & 153.16\% & 9.68\%     \\
	Q3      & 8.52\%   & 99.79\%    & 126.15\% & 8.54\%     & 126.61\% & 20.60\%    \\
	Q4      & 8.42\%   & 99.80\%    & 90.95\%  & 8.43\%     & 91.23\%  & 33.53\%    \\
	Average & 8.64\%   & 99.72\%    & 313.86\% & 8.67\%     & 315.11\% & 16.72\%    \\ \hline
\end{tabular}
\begin{tablenotes}
	\footnotesize
	\item[] L: Logistic
	\item[] DT: Decision Tree
	\item[] NB: Native Bayes
\end{tablenotes}
\end{threeparttable}
\end{table}

\end{document}